\documentclass[12pt]{jmlr} 

\usepackage[utf8]{inputenc} 
\usepackage[T1]{fontenc}    
\usepackage{hyperref}       
\usepackage{url}            
\usepackage{booktabs}       
\usepackage{amsfonts}       
\usepackage{nicefrac}       
\usepackage{microtype}      

\usepackage{amsmath}
\usepackage{amssymb}
\usepackage{comment}

\newlength{\minipagewidth}
\newcommand{\X}{\mathcal{X}}

\newcommand{\A}{\mathcal{A}}

\newcommand{\real}{\mathbb{R}}

\newcommand{\Xw}{\mathcal{X}}
\newcommand{\Aw}{\mathcal{A}}

\newcommand{\II}[1]{\mathbb{I}_{\left\{#1\right\}}}

\newcommand{\EE}[1]{\mathbb{E}\left[#1\right]}

\def\argmax{\mathop{\mbox{ arg\,max}}}
\newcommand{\ra}{\rightarrow}

\newcommand{\norm}[1]{\left\|#1\right\|}

\newcommand{\infnorm}[1]{\norm{#1}_\infty}
\newcommand{\ev}[1]{\left\{#1\right\}}
\newcommand{\pa}[1]{\left(#1\right)}

\newcommand{\wt}{\widetilde}
\newcommand{\transpose}{^\top}

\newcommand{\DD}[3]{D_{#3}\left(#1\middle\lVert#2\right)}
\newcommand{\DDC}[2]{\DD{#1}{#2}{C}}
\newcommand{\DDS}[2]{\DD{#1}{#2}{S}}

\newcommand{\trho}{\wt{\rho}}

\newcommand{\sa}{\mu} 
\newcommand{\sd}{\nu} 

\usepackage{todonotes}
\definecolor{PalePurp}{rgb}{0.66,0.57,0.66}

\newtheorem{assumption}{Assumption}

\title[A unified view of entropy-regularized MDPs]{A Unified View of Entropy-Regularized \\ Markov Decision Processes}

\author{\Name{Gergely Neu} \Email{gergely.neu@gmail.com}\\
 \addr Universitat Pompeu Fabra, Barcelona, Spain
 \AND
 \Name{Vicen\c c  G\'omez} \Email{vicen.gomez@upf.edu}\\
 \addr Universitat Pompeu Fabra, Barcelona, Spain
 \AND
 \Name{Anders Jonsson} \Email{anders.jonsson@upf.edu}\\
 \addr Universitat Pompeu Fabra, Barcelona, Spain
}

\begin{document}

\maketitle

\begin{abstract}
  We propose a general framework for entropy-regularized average-reward reinforcement learning in Markov decision processes (MDPs). 
  Our approach is based on extending the linear-programming formulation of policy optimization in MDPs to accommodate convex regularization 
functions. Our key result is showing that using the conditional entropy of the joint state-action distributions as regularization yields a 
dual optimization problem closely resembling the Bellman optimality equations. This result enables us to formalize a number of 
state-of-the-art entropy-regularized reinforcement learning algorithms as approximate variants of Mirror Descent or Dual Averaging, and thus
to argue about the convergence properties of these methods. In particular, we show that the exact version of the TRPO algorithm of 
\citet{schulman2015trust} actually converges to the optimal policy, while the entropy-regularized policy gradient methods of 
\citet{mnih2016asynchronous} may fail to converge to a fixed point. Finally, we illustrate empirically the effects of using various 
regularization techniques on learning performance in a simple reinforcement learning setup.

\end{abstract}

\section{Introduction}\label{sec:intro}
Reinforcement learning is the discipline of model-based optimal sequential decision-making in unknown stochastic environments. In 
average-reward reinforcement learning, the goal is to find a behavior policy that maximizes the long-term average reward, taking into 
account the effect of each decision on the future evolution of the decision-making process. In known environments, this 
optimization problem has been studied (at least) since the influential work of \citet{bellman57} and \citet{howard60}: the optimal 
behavior policy can be formulated as the solution of the Bellman optimality equations. In unknown environments with partially known or 
misspecified models, greedily solving these equations often results in policies that are far from optimal in the 
true environment. Rooted in statistical learning theory \citep{Vap13}, the notion of \emph{regularization} offers a principled way of 
dealing with this issue, among many others. In particular, \emph{entropy regularization} has proven to be one of the most 
successful tools of machine learning and related fields \citep{LW94,Vov90,FS97,KW01,ArHaKa12}. 

The idea of entropy regularization has also been used extensively in the reinforcement learning literature \citep{sutton,Sze10}. 
Entropy-regularized variants of the classic Bellman equations and the entailing reinforcement-learning algorithms have been proposed to 
induce safe exploration \citep{foxUAI16} and risk-sensitive policies \citep{How72,marcus1997risk,Rus10}, or to model observed behavior of 
imperfect decision-makers \citep{ZBD10,Zie10,braun2011path}, among others. Complementary to these approaches rooted in dynamic 
programming, another line of work proposes direct policy search methods attempting to optimize various entropy-regularized objectives 
\citep{WP91,peters2010relative,schulman2015trust,mnih2016asynchronous,o2017pgq}, with the main goal of driving a safe online exploration 
procedure
in an unknown Markov decision process. Notably, the state-of-the-art methods of \citet{mnih2016asynchronous} and \citet{schulman2015trust} 
are both based on entropy-regularized policy search.

In this work, we connect these two seemingly disparate lines of work by showing a strong Lagrangian duality between the entropy-regularized 
Bellman equations and a certain regularized average-reward objective. Specifically, we extend the linear-programming formulation of the 
problem of optimization in MDPs to accommodate convex regularization functions, resulting in a convex program. We 
show that using the \emph{conditional entropy} of the joint state-action distribution gives rise to a set of nonlinear equations resembling 
the Bellman optimality equations. Observing this duality enables us to establish a connection between regularized versions of value and 
policy iteration methods \citep{PS78} and incremental convex optimization methods like Mirror Descent \citep{NY83,BT03} or Dual Averaging 
\citep{xiao2010dual,mcmahan2014survey,Haz16,SS12}. For instance, the convex-optimization view we propose reveals that the TRPO algorithm of 
\citet{schulman2015trust} and the regularized policy-gradient method of \citet{mnih2016asynchronous} are approximate versions of 
Mirror Descent and Dual Averaging, respectively, and that both can be interpreted as regularized policy iteration methods. 

Our work provides a theoretical justification for various algorithms that were
first derived heuristically. In particular, our framework reveals that the
exact version of TRPO is \emph{identical} to the MDP-E algorithm of
\citet{even-dar09OnlineMDP}. This establishes the fact that the policy updates
of TRPO converge to the optimal policy, improving on the theoretical results
claimed by \citet{schulman2015trust}. We also argue that our formulation is
useful for pointing out possible inconsistencies of heuristic learning
algorithms. In particular, we show that the approximation steps employed by
\citet{mnih2016asynchronous} may break the convexity of the objective, thus
possibly leading to convergence to bad local optima or even divergence. This
observation is in accordance with the very recent findings of~\citet{AL16}, who
show that value iteration with poorly chosen approximate updates may lead to
divergence.  To complement these results, we suggest an alternative objective
that can be optimized consistently, avoiding the possibility of diverging.

A similar Lagrangian duality between the Bellman equations and entropy maximization has been previously noted by \citet[Sec.~5.2]{Zie10} 
and \citet{rawlik2012stochastic} for a special class of \emph{episodic} Markov decision processes where the time index within the episode 
is part of the state representation. In this particular setting, the convexity of the conditional entropy is more obvious. One of our key 
observations is pointing out the convexity of the conditional entropy of distribution functions defined over general state spaces, which 
enables us to develop a much broader theory of 
regularized Markov decision processes. We note that our theory also readily extends to discounted MDPs by replacing the stationary 
state-action distributions we consider by \emph{discounted state-action occupancy measures}. For consistency, we will discuss each 
particular algorithm in their most natural average-reward version, noting that all conclusions remain valid in the simpler discounted and 
episodic settings.

The rest of the paper is organized as follows. In Section~\ref{sec:prelim}, we provide background on average-reward Markov decision 
processes, briefly discussing both linear-programming and dynamic-programming derivations of the optimal control. In 
Section~\ref{sec:regmdp}, we provide a convex-programming formulation of regularized average-reward Markov decision processes, and show 
the connection to the regularized Bellman equations. Section~\ref{sec:dp} provides a brief summary of the complementary 
dynamic-programming formulation and discusses regularized equivalents of related concepts, such as expressions of the regularized policy 
gradient. In Section~\ref{sec:algs}, we describe several existing learning algorithms in our framework.
We provide an empirical evaluation of various regularization schemes in a simple reinforcement learning problem
in Section~\ref{sec:exp}. 

\paragraph{Notation.} Given a finite set $\mathcal{S}$, we will often use $\sum_s$ as shorthand for $\sum_{s\in\mathcal{S}}$, and we use 
$\Delta(\mathcal{S})=\{\mu\in\real^{\mathcal{S}}: \sum_s\mu(s)=1, \mu(s)\geq 0 \, \pa{\forall s}\}$ to denote the set of all 
probability 
distributions on $\mathcal{S}$.

\section{Preliminaries on Markov decision processes}\label{sec:prelim}
We consider a finite Markov decision process (MDP) $M = \pa{\Xw,\Aw,P,r}$, where $\X$ is the finite state space, $\A$ is the finite action 
space, $P:\X\times\A\times\X\ra[0,1]$ is the transition function, with $P(y|x,a)$ denoting the probability of moving to state $y$ from 
state 
$x$ when taking action $a$, and $r:\X\times\A\ra\mathbb{R}$ is the reward function mapping state-action pairs to rewards.

In each round $t$, the learner observes state $X_t\in\Xw$,
selects action $A_t\in\Aw$, moves to the next state $X_{t+1}\sim P(\cdot|X_t,A_t)$,
and obtains reward $r(X_t,A_t)$.
The goal is to select
actions as to maximize some notion of cumulative reward.
In this paper we consider the
\emph{average-reward} criterion
$\liminf_{T\ra \infty}\EE{\frac{1}{T}\sum_{t=1}^T r_t(X_t,A_t)}$.
A \emph{stationary state-feedback policy} (or \emph{policy} for short)
defines a probability distribution $\pi(\cdot|x)$ over the learner's actions in state $x$. 
MDP theory (see, e.g., \citet{Puterman:1994}) stipulates that under mild conditions, the average-reward criterion 
can be maximized by stationary policies.
Throughout the paper, we make the following mild assumption about the MDP:
\begin{assumption}\label{ass:unichain}
 The MDP $M$ is \emph{unichain}: All stationary policies $\pi$ induce a unique stationary distribution $\nu_\pi$ over 
the state space satisfying $\nu_\pi(y) = \sum_{x,a} P(y|x,a) \pi(a|x) \nu_{\pi}(x)$  for all $y\in\X$.
\end{assumption}
In particular, this assumption is satisfied if all policies induce an irreducible and aperiodic Markov chain \citep{Puterman:1994}. For 
ease 
of exposition in this section, we also make the following simplifying assumption:
\begin{assumption}\label{ass:recurrent}
 The MDP $M$ admits a single recurrent class: All stationary policies $\pi$ induce stationary distributions 
strictly supported  on the same set $\Xw'\subseteq\Xw$.
\end{assumption}
In general, this assumption is very restrictive in that it does not allow policies to cover different parts of the state space. We stress
that our results in the later sections \emph{do not} require this assumption to hold.
With the above assumptions in mind, we can define the average reward of any policy $\pi$ as
\[
 \rho(\pi) = \lim_{T\ra \infty} \EE{\frac{1}{T}\sum_{t=1}^T r_t(X_t,A_t)},
\]
where $A_t\sim\pi(\cdot|X_t)$ in each round $t$ and the existence of the limit is ensured by Assumption~\ref{ass:unichain}. Furthermore, the 
average reward of any policy $\pi$ can be 
simply written as $\rho(\pi) = \sum_{x,a} \nu_\pi(x) \pi(a|x) r(x,a)$, which is a linear function of the stationary state-action 
distribution $\mu_\pi=\nu_\pi \pi$. This suggests that finding the optimal policy can be equivalently written as a linear program (LP) where 
the decision 
variable is the stationary state-action distribution. Defining the set of all feasible stationary distributions as
\begin{equation}\label{eq:Delta}
 \Delta= \ev{\mu\in\Delta(\X\times\A): \sum_b \mu(y,b) = \sum_{x,a} P(y|x,a) \mu(x,a) \;\;\pa{\forall y}},
\end{equation}
the problem of maximizing the average reward can be written as
\begin{equation}\label{eq:LPdual}
 \mu^* = \argmax_{\mu\in\Delta} \rho(\mu).
\end{equation}
Just as a policy $\pi$ induces stationary distributions $\nu_\pi$ and $\mu_\pi$, a stationary distribution $\mu$ induces a state 
distribution $\nu_\mu$ defined as $\nu_\mu(x)=\sum_a\mu(x,a)$ and a policy $\pi_\mu$ defined as $\pi_\mu(a|x) = \mu(x,a)/\nu_\mu(x)$, where 
the denominator is strictly positive for recurrent states by Assumption~\ref{ass:recurrent}.
Since $\Delta$ is a compact polytope (non-empty by Assumption~\ref{ass:unichain}) the maximum in~\eqref{eq:LPdual} is well-defined 
and induces an optimal policy $\pi_{\mu^*}$ in recurrent states. Due to Assumption~\ref{ass:recurrent}, $\pi_{\mu^*}$ can be 
arbitrarily defined in transient states.

The linear program specified in Equation~\eqref{eq:LPdual} is well studied in the MDP literature (see, 
e.g., \citealp[Section 8.8]{Puterman:1994}), although most commonly as the dual of the linear program
 \begin{eqnarray}
 &\min_{\rho \in\real} &\rho \label{eq:LPprimal_obj}
 \\
 &\mbox{subject to} & \rho + V(x) - \sum_y P(y|x,a) V(y) \ge r(x,a), \qquad\forall(x,a). \label{eq:LPprimal_const}
 \end{eqnarray}
Here, the dual variables $V$ are commonly referred to as the \emph{value functions}.
 By strong LP duality and our Assumption~\ref{ass:unichain}, the solution to this LP equals the optimal average reward $\rho^*$ and the dual 
variables $V^*$ at the optimum
are the solution to the \emph{average-reward Bellman optimality equations}
\begin{equation}\label{eq:Bellman}
 V^*(x) = \max_a \pa{r(x,a) - \rho^* + \sum_y P(y|x,a) V^*(x)}, \qquad\pa{\forall x}.
\end{equation}
Note that $V^*$ is not unique as for any solution $V$, a constant shift $V - c$ for any $c\in\real$ is also a solution.
However, we can obtain a unique solution $V^*$ by imposing the additional constraint $\sum_{x,a} \mu^*(x,a) V^*(x)=0$, which states that the 
expected value should equal $0$.

\section{Regularized MDPs: A convex-optimization view}\label{sec:regmdp}
Inspired by the LP formulation of the average-reward optimization problem~\eqref{eq:LPdual}, we now define a regularized optimization 
objective---a framework that will lead us to our main results. Our results in this section only require the mild 
Assumption~\ref{ass:unichain}.
Our regularized optimization problem takes the form
\begin{align}
\max_{\sa\in\Delta} \;\; \trho_\eta(\mu) = \max_{\sa\in\Delta} \ev{\sum_{x,a}\sa(x,a)r(x,a) - \frac 1 \eta R(\sa)},\label{eq:primal}
\end{align}
where $R(\mu):\real^{\X\times\A} \ra \real$ is a convex regularization function and $\eta > 0$ is a \emph{learning rate} that trades off 
the original objective and regularization. Note that $\eta=\infty$ recovers the unregularized objective.
Unlike previous work on LP formulations for MDPs, we find it useful to regard \eqref{eq:primal} as the \emph{primal}.

We focus on two families of regularization functions: the \emph{negative Shannon entropy} of $(X,A)\sim \mu$,
\begin{equation}
 R_S(\mu) = \sum_{x,a} \mu(x,a) \log \mu(x,a),
\end{equation}
and the \emph{negative conditional entropy} of $(X,A)\sim \mu$,
\begin{equation}
 R_C(\mu) = \sum_{x,a} \mu(x,a) \log \frac{\mu(x,a)}{\sum_b \mu(x,b)} = \sum_{x,a} \nu_\mu(x) \pi_\mu(a|x) \log \pi_\mu(a|x).
\end{equation}
In what follows, we refer to these functions as the relative entropy and the conditional entropy. 
We also make use of the Bregman divergences induced by $R_S$ and $R_C$ which take the respective forms
\[
 \DDS{\mu}{\mu'} = \sum_{x,a} \mu(x,a) \log \frac{\mu(x,a)}{\mu'(x,a)}
\qquad \mbox{and} \qquad 
 \DDC{\mu}{\mu'} = \sum_{x,a} \mu(x,a) \log \frac{\pi_\mu(a|x)}{\pi_{\mu'}(a|x)}.
\]
While the form of $D_S$ is standard (it is the relative entropy between two state-action distributions), the fact that $D_C$ is the Bregman 
divergence of $R_C$ (or even that $R_C$ is convex) is not immediately obvious\footnote{In the special case of loop-free episodic 
environments, showing the convexity of $R_C$ is straightforward \citep{lafferty01conditional,Zie10,rawlik2012stochastic}.}. 
The following proposition asserts this statement, which we prove in Appendix~\ref{app:convexity}. The only work we are aware of that 
establishes a comparable result is the recent paper of~\citet{NG17}.
\begin{proposition}\label{prop:convexity}
 The Bregman divergence corresponding to the conditional entropy $R_C$ is $D_C$. Furthermore, $D_C$ is nonnegative on $\Delta$, implying  
that $R_C$ is convex and $D_C$ is convex in its first argument.
\end{proposition}

We proceed to derive the dual functions and optimal solutions to \eqref{eq:primal} for our two choices of regularization functions.
Without loss of generality, we assume that the reference 
policy $\pi_{\mu'}$ has full support, which implies that the corresponding stationary distribution $\mu'$ is strictly positive on the 
recurrent 
set $\Xw'$. We only provide the derivations for the Bregman divergences; the calculations are analogous for $R_S$ and $R_C$. Both of these 
solutions will be expressed with the help of dual variables $V:\real^\X \ra \real$ which are useful to think about as 
\emph{value functions}, as in the case of the LP formulation~\eqref{eq:LPdual}. We also define the 
corresponding \emph{advantage functions} $A(x,a) = r(x,a)+\sum_y P(y|x,a) V(y) - V(x)$.

\subsection{Relative entropy}
The choice $R=\DDS{\cdot}{\mu'}$ has been studied before by \citet{peters2010relative} and \citet{ZiNe13}; we defer the proofs to 
Appendix~\ref{app:shannon}.
The optimal state-action distribution for a given value of $\eta$ is
\begin{align}\label{eq:optsad}
\sa^*_\eta(x,a) &\propto \sa'(x,a) e^{\eta A^*_\eta(x,a)}, 
\end{align}
where $A^*_\eta$ is the advantage function for the optimal dual variables $V^*_\eta$. The dual function is
\begin{equation}
 g(V) =  \frac 1 \eta \log\sum_{x,a}\sa'(x,a)e^{\eta A(x,a)},\label{eq:reldual}
\end{equation}
that now needs to be minimized on $\real^\X$ with no constraints in order to obtain $V^*_\eta$. By 
strong duality, $g$ is convex in $V$ and takes the value $\trho_\eta^* = \max_{\mu\in\Delta} \trho_\eta(\mu)$ at its optimum. 
 
\subsection{Conditional entropy}
The choice $R=\DDC{\cdot}{\mu'}$ leads to our main contributions.
Similar to above, the optimal policy is
\begin{align}\label{eq:optpol}
\pi_\eta^*(a|x) \propto \pi_{\mu'}(a|x) e^{\eta A_\eta^*(x,a)}.
\end{align}
In this case, the dual problem closely resembles the average-reward Bellman optimality equations~\eqref{eq:Bellman}:
\begin{proposition}\label{prop:conddual}
 The dual of the optimization problem \eqref{eq:primal} when $R = \DDC{\cdot}{\mu'}$ is given by
 \begin{eqnarray*}
 &\min_{\lambda \in\real} & \lambda \label{eq:CRdual_obj}
 \\
 &subject\;to & V(x) = \frac 1 \eta \log\sum_a\pi_{\mu'}(a|x)\exp\pa{\eta\pa{r(x,a) - \lambda + \sum_y P(y|x,a) V(y)}},\;\pa{\forall 
x}. \label{eq:CRdual_const}
 \end{eqnarray*}
\end{proposition}
We defer the proofs to Appendix~\ref{app:condentropy}. Using strong duality, the optimum of the above problem is $\trho_\eta^*$, which 
implies 
that the optimal dual variables $V_\eta^*$ are given as 
a solution to the system of equations
\begin{equation}\label{eq:regBOE}
 V_\eta^*(x) = \frac 1 \eta \log\sum_a\pi_{\mu'}(a|x)\exp\pa{\eta\pa{r(x,a) - \trho_\eta^* + \sum_y P(y|x,a) V_\eta^*(y)}}, \quad\pa{\forall 
x}.
\end{equation}
By analogy with the Bellman optimality equations~\eqref{eq:Bellman}, we call this the \emph{regularized average-reward Bellman 
optimality equations}. Since $\trho_\eta^*$ is guaranteed to be finite (because it is the maximum of a bounded function on a compact 
domain), the solution to the above optimization problem is well-defined, bounded, and unique up to a constant shift (as in the 
case of the LP dual variables).
Again, we can make the solution unique by imposing the constraint that the expected value should equal $0$.

\section{Dynamic programming in regularized MDPs}\label{sec:dp}

We now present a dynamic-programming view of the regularized optimization problem \citep{Ber07:DPbookVol2} for the choice 
$R=\DDC{\cdot}{\mu'}$. Similar 
derivations have been done several times for \emph{discounted and episodic 
MDPs} \citep{LS96,Rus10,azar2011dynamic,rawlik2012stochastic,AL16,foxUAI16}, but we are not aware of any work that considers the 
average-reward case. That said,
the generalization is straightforward, and the existence and unicity of the optimal solution to the Bellman optimality 
equation~\eqref{eq:regBOE} follows from our results in the previous section.

We first define the \emph{regularized Bellman equations} for an arbitrary policy $\pi$ and a reference policy $\pi'$:
\begin{equation}\label{eq:regBE}
 V_\eta^\pi(x) = \sum_a \pi(a|x) \pa{r(x,a) - \frac 1 \eta \log \frac{\pi(a|x)}{\pi'(a|x)} - \trho_\eta(\pi)+ \sum_y P(y|x,a) 
V_\eta^\pi(y)} 
\quad\pa{\forall x},
\end{equation}
where $\trho_\eta(\pi)$ is the regularized average reward of policy $\pi$ defined as in Equation~\eqref{eq:primal}.
By our Assumption~\ref{ass:unichain} and Proposition~4.2.4~of~\citet{Ber07:DPbookVol2}, it is easy to show that this system of equations 
has a unique solution satisfying the additional constraint $\sum_{x,a} \mu_\pi(x,a) V_\eta^\pi(x) = 0$.
We also define the \emph{Bellman optimality operator} $T_\eta^{*|\pi'}$ and the \emph{Bellman operator} $T_\eta^{\pi|\pi'}$ that correspond 
to the Bellman equations~\eqref{eq:regBOE} and \eqref{eq:regBE}, respectively, as well as the greedy policy operator $G_\eta^{\pi'}$ that 
corresponds to Equation~\eqref{eq:optpol} (for completeness, the formal definitions appear in Appendix~\ref{app:BOP}).

We include two results that are useful for deriving approximate dynamic programming algorithms. We first 
provide a counterpart to the \emph{performance-difference lemma} (\citealp[Prop.~1]{BK97}, 
\citealp[Lemma~6.1]{KakadeLangford2002}, \citealp{Cao07}). This statement will rely on the regularized advantage function $A^\pi_\eta$ 
defined for each policy $\pi$ as
\[
 A^\pi_\eta(x,a) = r(x,a) - \frac 1 \eta \frac{\log \pi(a|x)}{\log \pi'(a|x)} - \trho_\eta(\pi) + \sum_y P(y|x,a) V_\eta^\pi(y) - 
V_\eta^\pi(x),
\]
where $V_\eta^\pi$ is the regularized value function corresponding to $\pi$ with baseline $\pi'$.
\begin{lemma}\label{lem:PD}
 For any pair of policies $\pi,\pi'$, we have
 \[
  \trho(\pi') - \trho(\pi) = \sum_{x,a} \mu_{\pi'}(x,a) A_{\eta}^\pi(x,a).
 \]
\end{lemma}
For completeness, we provide the simple proof in Appendix~\ref{app:PD}. 

Second, we provide an expression for the gradient of $\trho_\eta$, thus providing a regularized counterpart of the \emph{policy gradient 
theorem} of \citet{SMSM99}. To formalize this statement, let us consider a policy 
$\pi_\theta$ parametrized by a vector $\theta\in\real^d$ and assume that the gradient $\nabla \pi_\theta(a|x)$ exists for all $x,a$ and 
all $\theta$. The form of the policy gradient is given by the following lemma, which we prove in Appendix~\ref{app:relgrad}:
\begin{lemma}\label{lemma:polgrad}
Assume that $\frac{\partial \pi_{\theta}(a|x)}{\partial \theta_i} / \pi_\theta(a|x) > 0$ for all $\theta_i,x,a$. The gradient of 
$\trho_\eta$ exists and satisfies
\[
 \nabla \trho_\eta(\theta) = \sum_{x,a} \mu_{\pi_\theta}(x,a) \nabla \log \pi_\theta(a|x) A_\eta^{\pi_{\theta}} (x,a).
\]
\end{lemma}

\section{Algorithms}\label{sec:algs}
In this section we derive several reinforcement learning algorithms based on our results. For clarity of presentation, we assume that the 
MDP $M$ is fully known, an assumption that we later relax in the experimental evaluation. We will study a generic sequential optimization 
framework where a sequence of policies $\pi_{k}$ are computed iteratively.
Inspired by the online convex optimization literature (see, e.g., \citealp{SS12,Haz16}) and by our convex-optimization formulation, we 
study two families of algorithms: Mirror Descent and Dual Averaging (also known as Follow-the-Regularized-Leader).

\subsection{Iterative policy optimization by Mirror Descent}\label{sec:MD}
A direct application of the Mirror Descent algorithm \citep{NY83,BT03,Martinet1978,R76} to our case is defined as
\begin{equation}\label{eq:MD}
 \mu_{k+1} = \argmax_{\mu\in\Delta} \ev{\rho(\mu) - \frac 1\eta \DD{\mu}{\mu_k}{R}},
\end{equation}
where $D_R$ is the Bregman divergence associated with the convex regularization function $R$.
We now proceed to show how various learning algorithms can be recovered from this formulation.
\subsubsection{Mirror Descent with the relative entropy}
We first remark that the Relative Entropy Policy Search (REPS) algorithm of \citet{peters2010relative} can be formulated as an instance of 
Mirror Descent with the Bregman divergence $D_S$.
This is easily seen by 
comparing the form of the update rule~\eqref{eq:MD} with the problem formulation of \citet[pp.~2]{peters2010relative}, with the slight 
difference that our regularization is additive and theirs is enforced as a constraint. It is easy to see that this only amounts to a change 
in learning rate. This connection is not new: it has been first shown by 
\citet{ZiNe13}\footnote{Although they primarily referred to Mirror Descent as the ``Proximal Point Algorithm'' following 
\citep{R76,Martinet1978}.}, and has been recently rediscovered by \citet{ML16}. 
Independently of each other, \citet{ZiNe13} and \citet{DGS14} both show that Mirror Descent achieves near-optimal regret guarantees in an 
online learning setup where the transition function is known, but the reward function is allowed to change arbitrarily between decision 
rounds. This implies that REPS duly converges to the optimal policy in our setup.

\subsubsection{Mirror Descent with the conditional entropy}\label{sec:MD_CR}
We next show that the Dynamic Policy Programming (DPP) algorithm of
\citet{AGK12} and the Trust-Region Policy Optimization 
(TRPO) algorithm of \citet{schulman2015trust}
are both approximate variants of Mirror Descent with the Bregman divergence $D_C$.
To see this, note that a full Mirror Descent update requires computing the optimal value function $V^*_\eta$ for the baseline $\mu_k$, 
e.g.~by regularized value iteration or regularized policy 
iteration (see Appendix~\ref{app:BOP}).
Since a full update for $V^*_\eta$ is expensive, DPP and TRPO provide two ways to approximate it. 
We remark that the algorithm of \citet{rawlik2012stochastic} can also be viewed as an instance of Mirror Descent for the finite-horizon 
episodic setting, in which the exact update can be computed efficiently by dynamic programming.

\paragraph{Dynamic Policy Programming.} We first claim that each iteration of DPP is a \emph{single regularized value iteration 
step}: Starting from the previous value function $V_k$, it extracts the greedy policy $\pi_{k+1} = 
G_\eta^{\pi_k}[V_k]$ and applies 
the Bellman optimality operator $T_\eta^{*|\pi_k}$ to obtain $V_{k+1} = T_\eta^{*|\pi_k}[V_k]$. This follows from comparing the form of DPP 
presented 
in Appendix~A of \citet{AGK12}: their update rules~(19) and~(20) precisely match the discounted analogue of our 
expressions~\eqref{eq:VI} in Appendix~\ref{app:BOP} with $\pi' = \pi_k$. The
convergence guarantees proved by \citet{AGK12} 
demonstrate the soundness of this approximate update.

\paragraph{Trust-Region Policy Optimization.} Second, we claim that each iteration of TRPO is a \emph{single policy iteration step}: TRPO 
first fully evaluates the policy $\pi_k$ to 
compute its \emph{unregularized} value function $V_k=V_{\infty}^{\pi_k}$ and then extracts the regularized greedy policy $\pi_{k+1} = 
G_\eta^{\pi_k}[V_k]$ with $\pi_k$ as a baseline. This can be seen by inspecting the TRPO update\footnote{As in the case of REPS, we 
discuss here the additive-regularization version of the algorithm. The entropy-constrained update actually implemented by 
\citet{schulman2015trust} only differs in the learning rate.} that takes the form
\[
 \pi_{k+1} = \argmax_{\pi} \ev{\sum_{x} \nu_{\pi_k}(x) \sum_a \pi(a|x) \pa{A_{\infty}^{\pi_k}(x,a) - \frac 1\eta \log 
\frac{\pi(a|x)}{\pi_k(a|x)}}}.
\]
This objective approximates Mirror Descent
by ignoring the effect of changing the policy on 
the state distribution. Surprisingly, using our formalism, this update can be expressed in 
closed form as
\[
 \pi_{k+1}(a|x) \propto \pi_{k}(a|x) e^{\eta A_{\infty}^{\pi_k}(x,a)}.
\]
We present the detailed derivations in Appendix~\ref{app:TRPO}.
A particularly interesting consequence of this result is that TRPO is \emph{completely equivalent} to the MDP-E algorithm of 
\citet{even-dar09OnlineMDP} (see also 
\citep{neu10o-ssp,neu14o-mdp-full}), which is known to minimize regret in an online setting, thus implying that TRPO also converges to the 
optimal policy in the stationary setting. This guarantee is much stronger than the ones provided by \citet{schulman2015trust}, who only 
claim that TRPO produces a monotonically improving sequence of policies (which may still converge to a suboptimal policy).

\subsection{Iterative policy optimization by Dual Averaging}\label{sec:DA}
We next study algorithms arising from the Dual Averaging scheme~\citep{xiao2010dual,mcmahan2014survey}, commonly known as 
Follow-the-Regularized-Leader in online learning~\citep{SS12,Haz16}. This algorithm is defined by the iteration
\begin{equation}\label{eq:DA}
 \mu_{k+1} = \argmax_{\mu\in\Delta} \ev{\rho(\mu) - \frac{1}{\eta_k} R(\mu)},
\end{equation}
where $\eta_k$ is usually an increasing sequence to ensure convergence in the limit. We are unaware of any pure instance of dual 
averaging using relative entropy, and only discuss conditional entropy below.

\subsubsection{Dual Averaging with the conditional entropy}
Just as for Mirror Descent, a full update~\eqref{eq:DA} requires computing the optimal value function $V_\eta^*$.
Various approximations of this update have been long studied in the RL
literature---see, e.g., \citep{LS96} (with additional discussion by
\citep{AL16}), \citep{perkins,Rus10,PS12,foxUAI16}. In this section, we focus
on the state-of-the-art algorithms of \citet{mnih2016asynchronous} and
\citet{o2017pgq} that were originally derived from an optimization formulation
resembling our Equation~\eqref{eq:primal}. Our main insight is that this
algorithm can be adjusted to have a dynamic-programming interpretation and a
convergence guarantee.

\paragraph{Entropy-regularized policy gradients.} The A3C algorithm of \citet{mnih2016asynchronous} aims to maximize
\[
 \rho(\pi) - \frac {1}{\eta_k} \sum_{x} \nu_{\pi_k}(x)  \sum_a \pi(a|x) \log \pi(a|x)
\]
by taking policy gradient steps.
Interestingly, our formalism implies a connection between TRPO and A3C.
Due to Lemma~\ref{lemma:polgrad}, the gradient of $\rho(\pi_\theta)$ with respect to $\theta$ coincides with the gradient of 
$\sum_{x} \nu_{\pi_k}(x) \sum_a \pi_\theta(a|x) A_{\infty}^{\pi_{k}}(x,a)$, so A3C actually attempts 
to optimize the objective
\begin{equation}\label{eq:RPG}
 \sum_{x} \nu_{\pi_k}(x) \sum_a \pi_\theta(a|x) \pa{A_{\infty}^{\pi_k}(x,a) - \frac {1}{\eta_k} \log \pi_\theta(a|x)}.
\end{equation}
This objective can be seen as the dual-averaging counterpart of the TRPO objective.
As in the case of TRPO, the maximizer of this objective can be computed in closed form as
\[
 \pi_{k+1} (a|x) \propto e^{\eta_k A_\infty^{\pi_k}}.
\]
Unlike for TRPO, we are not aware of any convergence results for A3C, and we believe the algorithm does not converge.
Indeed, the objective~\eqref{eq:RPG} is non-convex in either of the natural parameters $\mu$ or $\pi$, which 
can cause premature convergence to a bad local optimum.
An even more serious concern is that the objective function changes between iterations, so gradient descent may fail to converge to 
\emph{any} stationary point. This problem is 
avoided by TRPO since the sum of the TRPO objectives is a sensible optimization objective~\cite[Theorem~4.1]{even-dar09OnlineMDP}. However, 
there is no such clear interpretation for the objective~\eqref{eq:RPG}. 

\citet[Section~3.1]{o2017pgq} study the stationary points of the objective~\eqref{eq:RPG} and, similarly to us, 
show a connection between a certain type of value function and a policy achieving the stationary point. However, they do not show that 
this stationary point is unique nor that gradient descent converges to a stationary point\footnote{Strictly speaking, 
\citet{o2017pgq} do not even show that a stationary point~\eqref{eq:RPG} exists.}. As we argue above, this may very well not be 
the case.  These observations are consistent with those of \citet{AL16}, who show that softmax policy updates may lead to inconsistent 
behavior when used in tandem with unregularized advantage functions.

To overcome these issues, we advocate for directly optimizing the objective~\eqref{eq:DA} instead of~\eqref{eq:RPG} via gradient descent. 
Due to the fact that~\eqref{eq:DA} is convex in $\mu$ and to standard results regarding dual averaging \citep{mcmahan2014survey}, this 
scheme is 
guaranteed to converge to the optimal policy. Estimating the gradients can be done analogously as for the unregularized objective, by our 
Lemma~\ref{lemma:polgrad}.

\section{Experiments}\label{sec:exp}
In this section we analyze empirically several of the algorithms described in the previous section, with the objective of 
illustrating the interplay of regularization and model-estimation error in a simple reinforcement learning setting. We consider an
iterative setup where in each episode $k=1,2,\dots,N$, we execute a policy $\pi_k$, observe the sample transitions and update the 
estimated model via maximum likelihood.  We focus on the regularization aspect,
with no other approximation error than that introduced by model estimation.  It is important
to emphasize that the comparison may not extend to other variants of the
algorithms or in the presence of other sources of approximation.
\begin{figure}[!tbp]
\centering
\includegraphics[width=\textwidth]{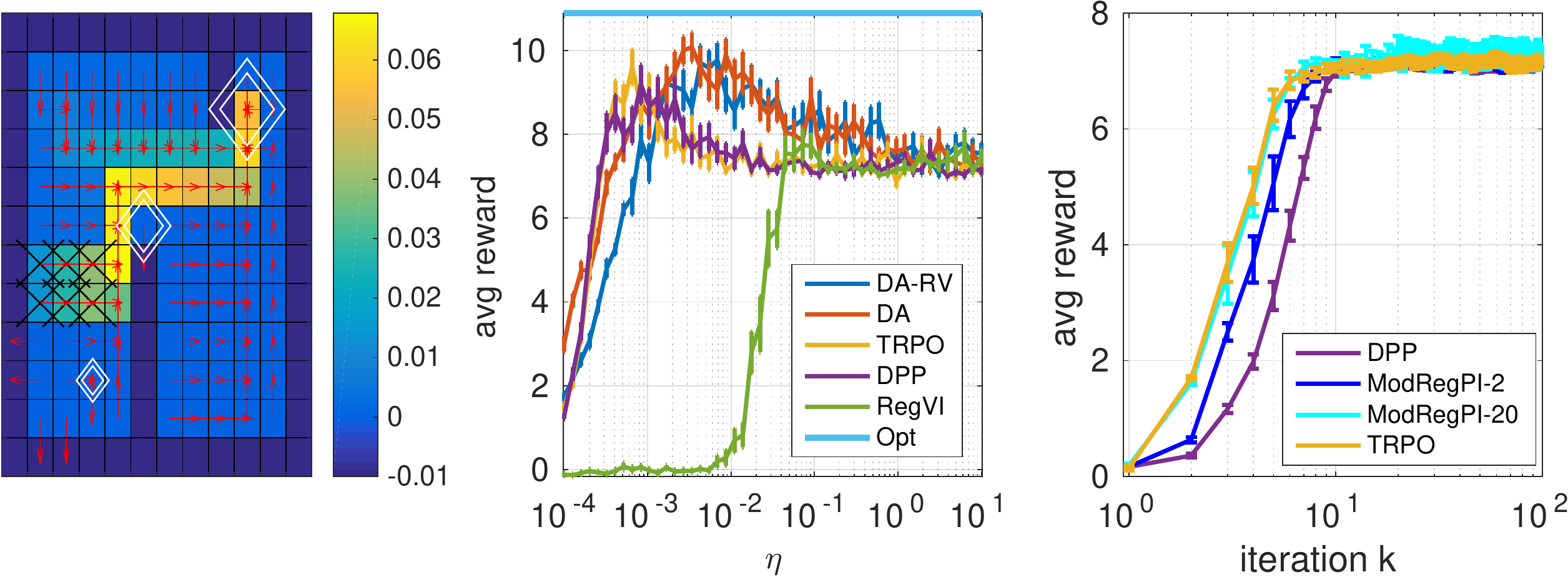}
\caption{
\textbf{{Left}}: the MDP used for evaluation. Reward is $-0.1$ at the walls and $5, 50, 200$ at the diamonds. The optimal policy is 
indicated by red
arrows. The cell colors correspond to the stationary state distribution
for open locations.
\textbf{{Middle}}: 
Average reward as a function of the learning rate $\eta$ for all algorithms
(see text for details).
Number of iterations $N$ and samples per iteration $S$ are $N=S=500$.
Results are taken over $20$ random runs per value of $\eta$.
\textbf{{Right}}: 
Performance of DPP, TRPO and two version of modified regularized Policy Iteration
for a fixed $\eta\approx 0.1$.
}
\label{fig:MDP}
\end{figure}

We consider a simple MDP, defined on a grid (Fig.~\ref{fig:MDP}, left), where an agent has four possible actions (up,
down, left and right) that succeed with probability $0.9$ but fail with
probability $0.1$. In case of failure, the agent does not move or goes to any
random adjacent location.
Negative (or positive) rewards are given after hitting a wall (or reaching one of the
white diamond locations, respectively). In both cases, the agent is sent back to one of the starting locations
(marked with 'X' in the figure).

The reward of the diamonds is proportional to the distance
from the starting locations. Therefore, the challenge of this experiment is to discover the path towards the top-right reward 
while learning the dynamics incrementally, and then exploit it.
Note that the optimal agent ignores the intermediate reward at the center and even prefers to
hit a wall in locations too far away from the largest reward (bottom-left).

We fix the number of iterations $N$ and samples per iteration $S$ and 
analyze the average reward of the final policy as a function of $\eta$.  
We compare the following algorithms: regularized Value Iteration with a fixed
reference uniform policy and fixed $\eta$ (RegVI); several variants of
approximate Mirror Descent, including DPP and TRPO (Section~\ref{sec:MD}); and
two Dual Averaging methods (DA and DA-RV). DA corresponds to optimizing
the objective~\eqref{eq:RPG}, which is not guaranteed to lead to an optimal policy, and DA-RV
corresponds to the iteration~\eqref{eq:DA}, which has convergence guarantees (Section~\ref{sec:DA}).
For both variants, we use a linear annealing schedule $\eta_k = \eta\cdot
k$.

Fig.~\ref{fig:MDP} (middle) shows results as a function of $\eta$.  The maximum
reward is depicted in blue at the top.  For very small $\eta$ (strong
regularization), all algorithms perform poorly and do not even reach the
intermediate reward.  In contrast, for very large $\eta$, they converge
prematurely to the greedy policy that exploits the intermediate reward.
Typically, for an intermediate value of $\eta$, the algorithms occasionally
discover the optimal path and exploit it.  Note that this is not the case for
RegVI, which never obtains the optimal policy. This shows that using both a
fixed value of $\eta$ and a fixed reference policy is a bad choice.  In this
MDP, we observe that the performance of both Dual Averaging methods (DA and
DA-RV) is very similar, and in general slightly better than the approximate Mirror Descent variants.

We also show an interesting relationship between the Mirror Descent approximations. Our
analysis in Section~\ref{sec:MD_CR} suggests an entire array of algorithms
lying between DPP and TRPO, just as Modified Policy Iteration
lies between Value Iteration and Policy Iteration~\citep{PS78,SGGG12}.
Fig.~\ref{fig:MDP} (right) illustrates this idea, showing  the convergence
of DPP and TRPO for a fixed value of $\eta$.  TRPO tends to converge faster
than DPP to a locally optimal policy, since DPP uses a single value
update per iteration.
Using more value updates leads to a modified regularized Policy
Iteration algorithm (we call it ModRegPI-2 and ModRegPI-20, for 2 and 20 updates,
respectively) that interpolates between DPP and TRPO.

\section{Conclusion}\label{sec:conc}
We have presented a unifying view of entropy-regularized MDPs from a convex-optimization perspective. We believe that such unifying 
theories 
can be very useful in moving a field forward: We recall that in the field of online learning theory, the convex-optimization view has 
enabled a unified treatment of many existing algorithms and acts today as the primary framework for deriving new algorithms (see the 
progress from \citet{CBLu06:book} through \citet{SS12} to \citet{Haz16}). 
In this paper, we argued that the convex-optimization view may also be very useful in analyzing algorithms for reinforcement learning: In 
particular, we demonstrated how this framework can be used to provide theoretical justification for state-of-the-art reinforcement 
learning algorithms, and how it can highlight potential problems with them. 
We expect that this newly-found connection will also open the door for constructing more advanced reinforcement 
learning algorithms by borrowing further ideas from the convex optimization literature, such as Composite Objective Mirror 
Descent~\cite{duchi2010composite} and Regularized Dual Averaging~\citep{xiao2010dual}.

Finally, we point out that our work \emph{does not} provide a statistical justification for using entropy regularization in reinforcement 
learning. In the case of online learning in known Markov decision processes with changing reward functions, entropy-regularization has 
been known to yield near-optimal learning algorithms \citep{even-dar09OnlineMDP,neu10o-ssp,neu14o-mdp-full,ZiNe13,DGS14}. It remains to 
be seen if this technique also provably helps in driving the exploration process in unknown Markov decision processes.

\bibliographystyle{abbrvnat}
\bibliography{ftrl,ngbib,allbib,shortconfs}

\appendix
\section{Complementary Technical Results}

\subsection{Convexity of the negative conditional entropy}
\label{app:convexity}

Let us consider the joint state-action distribution $\mu$ on the finite set $\X\times\A$. We denote $\nu_\mu(x) = \sum_a \mu(x,a)$ and 
$\pi_\mu(a|x) = \mu(x,a)/\nu_\mu(x)$ for all $x,a$. We study the negative conditional entropy of $(X,A)\sim\mu$ as a function of $\mu$:
\[
\begin{split}
R_C(\mu) 
 &= \sum_{x,a} \mu(x,a) \log\frac{\mu(x,a)}{\sum_b \mu(x,b)}
= \sum_{x,a} \mu(x,a) \log\frac{\mu(x,a)}{\nu_\mu(x)} .
\end{split}
\]
We will study the Bregman divergence $D_{R_C}$ corresponding to $R_C$:
\[
 \DD{\mu}{\mu'}{R_C} = R_C(\mu) - R_C(\mu') - \nabla R_C(\mu') \transpose (\mu - \mu'),
\]
where the inner product between two vectors $v,w\in\real^{\X\times\A}$ is $w\transpose v = \sum_{x,a} v(x,a) w(x,a)$.
Our aim is to show that $D_{R_C}$ is nonnegative, which will imply the convexity of $R_C$.

We begin by computing the partial derivative of $R_C(\mu)$ with respect to $\mu(x,a)$:
\[
 \frac{\partial R_C(\mu)}{\partial \mu(x,a)} = \log\frac{\mu(x,a)}{\nu_\mu(x)} + 1 - \sum_b\frac{\mu(x,b)}{\nu_\mu(x)} = 
\log\frac{\mu(x,a)}{\nu_\mu(x)},
\]
where we used the fact that $\partial \nu_\mu(x)/\partial \mu(x,a) = 1$ for all $a$.
With this expression, we have
\[
\begin{split}
 R_C(\mu') + \nabla R_C(\mu') \transpose (\mu - \mu')
 &= \sum_{x,a} \mu'(x,a) \log\frac{\mu'(x,a)}{\nu_{\mu'}(x)} + \sum_{x,a} \pa{\mu(x,a) - \mu'(x,a)} \log\frac{\mu'(x,a)}{\nu_{\mu'}(x)}
 \\
 &= \sum_{x,a} \mu(x,a) \log\frac{\mu'(x,a)}{\nu_{\mu'}(x)}.
\end{split}
\]
Thus, the Bregman divergence takes the form
\[
\begin{split}
 \DD{\mu}{\mu'}{R_C} &= \sum_{x,a} \mu(x,a) \pa{\log\frac{\mu(x,a)}{\nu_\mu(x)} - \log\frac{\mu'(x,a)}{\nu_{\mu'}(x)}}
 \\
 &= \sum_{x,a} \mu(x,a) \log\frac{\pi_\mu(a|x)}{\pi_{\mu'}(a|x)}
 = \sum_{x} \nu(x) \sum_a \pi_\mu(a|x) \log\frac{\pi_\mu(a|x)}{\pi_{\mu'}(a|x)}.
\end{split}
\]
This proves that the Bregman divergence corresponding to $R_C$ coincides with $D_C$, as claimed. To conclude the proof, note that $D_C$ is 
the average relative entropy between the distributions $\pi_\mu$ and $\pi_{\mu'}$---that is, a sum a positive terms. Indeed, this shows that 
$D_C$ is 
nonnegative on the set of state-action distributions $\Delta(\X\times\A)$, proving that $R_C(\mu)$ is convex. \jmlrQED

\subsection{Derivation of optimal policies} \label{app:duals}
Here we prove the results stated in Equations~\eqref{eq:optsad}-\eqref{eq:optpol} and Proposition~\ref{prop:conddual}, which give the 
expressions for the dual optimization problems and the optimal solutions corresponding to the primal optimization 
problem~\eqref{eq:primal}, 
for the two choices of regularization function $\DDS{\cdot}{\mu'}$ and $\DDC{\cdot}{\mu'}$. We start with generic derivations that will be 
helpful for analyzing both cases and then turn to studying the individual regularizers.

Recall that the primal optimization objective in \eqref{eq:primal} is given by
\begin{align*}
\max_{\sa\in\Delta} \;\; \trho_\eta(\mu) = \max_{\sa\in\Delta} \ev{\sum_{x,a}\sa(x,a)r(x,a) - \frac 1 \eta R(\sa)},
\end{align*}
where $\Delta$, the feasible set of stationary distributions, is defined by the following constraints:
\begin{align}
\sum_b\sa(y,b) &= \sum_{x,a}\sa(x,a)P(y|x,a), \;\; \forall y\in\Xw,\label{eq:balance}\\
\sum_{x,a}\sa(x,a) &= 1,\label{eq:distr}\\
\sa(x,a) &\geq 0, \;\; \forall (x,a)\in\Xw\times\Aw.\label{eq:lzero}
\end{align}
We begin by noting that for all state-action pairs where $\mu'(x,a) = 0$, the optimal solution $\mu^*_\eta(x,a)$ will also be zero, thanks 
to the form of our regularized objective. Thus, without loss of generality, we will assume that all states are recurrent under 
$\mu'$: $\mu'(x,a)>0$ holds for all state-action pairs.

For any choice of regularizer $R$, the Lagrangian of the primal~\eqref{eq:primal} is given by
\begin{align*}
\mathcal{L}(\sa;V,&\lambda,\varphi) = \sum_{x,a}\sa(x,a)r(x,a) - \frac 1 \eta R(\sa) + \sum_y V(y)\pa{\sum_{x,a}\sa(x,a)P(y|x,a) - 
\sum_b\sa(y,b)}\\
 & \hspace*{.75cm} + \lambda\pa{1 - \sum_{x,a}\sa(x,a)}+\sum_{x,a}\varphi(x,a)\sa(x,a)\\
 &= \sum_{x,a}\sa(x,a)\pa{r(x,a) + \sum_yP(y|x,a)V(y) - V(x)-\lambda+\varphi(x,a)} - \frac 1 \eta R(\sa) + \lambda\\
 &= \sum_{x,a}\sa(x,a)\pa{A(x,a)-\lambda+\varphi(x,a)} - \frac 1 \eta R(\sa) + \lambda,
\end{align*}
where $V$, $\lambda$ and $\varphi$ are the Lagrange multipliers\footnote{Technically, these are KKT multipliers as we also have inequality 
constraints. However, these will be eliminated by means of complementary slackness in the next sections.}, and $A$ is the advantage 
function for $V$. Setting the gradient of the 
Lagrangian with respect to $\mu$ to $0$ yields the system of equations
\begin{align}
0 = \frac{\partial \mathcal{L}}{\partial\sa(x,a)}&= \pa{A(x,a) - \lambda + \varphi(x,a)} - \frac 1 \eta 
\frac{\partial R(\sa)}{\partial\sa(x,a)},\nonumber\\
\Leftrightarrow \;\; \frac{\partial R(\sa)}{\partial\sa(x,a)}&= \eta\pa{A(x,a) - \lambda + 
\varphi(x,a)},\label{eq:gradient}
\end{align}
for all $x,a$. By the first-order stationary condition, the unique optimal solution $\sa_\eta^*$ satisfies this system of equations.
To obtain the final solution we have to compute the optimal values $V_\eta^*$, $\lambda_\eta^*$ and $\varphi_\eta^*$ of the Lagrange 
multipliers by optimizing the dual optimization objective $g(V,\lambda,\varphi)=\mathcal{L}(\sa_\eta^*;V,\lambda,\varphi)$, and insert into 
the expression for $\sa_\eta^*$. $V$ and $\lambda$ are unconstrained in the dual, while $\varphi$ satisfies $\varphi(x,a)\geq 0$ for each 
$(x,a)\in\X\times\A$. We give the derivations for each regularizer below.

\subsection{The relative entropy}\label{app:shannon}
Here we prove the results for $R(\sa)=\DDS{\sa}{\sa'} = \sum_{x,a} \mu(x,a) \log \frac{\mu(x,a)}{\mu'(x,a)}$. The gradient of $R$ is
\begin{align*}
\frac{\partial R(\sa)}{\partial\sa(x,a)} &= \log\frac{\sa(x,a)}{\sa'(x,a)} + 1.
\end{align*}
The optimal state-action distribution $\sa_\eta^*$ is now directly given by Equation~\eqref{eq:gradient}:
\begin{align}
\sa_\eta^*(x,a) &= \sa'(x,a)\exp\pa{\eta \pa{A(x,a) - \lambda + \varphi(x,a)}-1}.\label{eq:relrelopt}
\end{align}
For $\sa_\eta^*$ to belong to $\Delta$, it has to satisfy Constraints~\eqref{eq:distr} and \eqref{eq:lzero}. Because of the exponent 
in~\eqref{eq:relrelopt}, $\sa_\eta^*(x,a)\geq 0$ trivially holds for any choice of $\varphi(x,a)$, and complementary slackness implies 
$\varphi_\eta^*(x,a)=0$ for each $(x,a)$. Eliminating $\varphi$ and inserting $\sa_\eta^*$ into Constraint~\eqref{eq:distr} gives us
\begin{align}
1 &= \sum_{x,a}\sa'(x,a)\exp\pa{\eta A(x,a)}e^{-\eta\lambda-1},\nonumber\\
\Leftrightarrow \;\; \lambda &= \frac 1 \eta \pa{ \log \sum_{x,a}\sa'(x,a)\exp\pa{\eta A(x,a)} - 1}.\label{eq:rellambda}
\end{align}
Since the value of $\lambda$ is uniquely determined by \eqref{eq:rellambda}, we can optimize the dual over $V$ only. The dual function is 
given by
\begin{align*}
g(V)=\mathcal{L}(\sa_\eta^*;V,\lambda) &= \sum_{x,a}\sa_\eta^*(x,a)\pa{A(x,a)-\lambda-\frac 1 \eta \log \frac{\mu_\eta^*(x,a)}{\mu'(x,a)}} 
+ \lambda = \frac 1 \eta + \lambda\nonumber\\
 &=\frac 1 \eta \log \sum_{x,a}\sa'(x,a)\exp\pa{\eta A(x,a)}.
\end{align*}
This is precisely the dual given in Equation~\eqref{eq:reldual}. Note that this dual function has no associated constraints. The expression 
for the optimal state-action distribution in Equation~\eqref{eq:optsad} is obtained by inserting the advantage function $A_\eta^*$ 
corresponding to the optimal value function $V_\eta^*$ into~\eqref{eq:relrelopt}.

\subsection{The conditional entropy}\label{app:condentropy}
We next prove the results for $R(\mu)=\DDC{\mu}{\mu'} = \sum_{x,a} \mu(x,a) \log \frac{\pi_\mu(a|x)}{\pi_{\mu'}(a|x)}$. The gradient of $R$ 
is
\begin{align*}
\frac{\partial R(\sa)}{\partial\sa(x,a)} &= \log\frac{\pi_\sa(a|x)}{\pi_{\mu'}(a|x)} + 
\sum_b\frac{\sa(x,b)}{\pi_\sa(b|x)}\cdot\frac{\partial\pi_\sa(b|x)}{\partial\sa(x,a)}.
\end{align*}
Since the policy is defined as $\pi_\mu(a|x)=\frac {\mu(x,a)} {\sd_\mu(x)}$, its gradient with respect to $\mu$ is
\begin{align*}
\frac{\partial\pi_\sa(b|x)}{\partial\sa(x,a)} = \frac {\II{a=b}} {\sd_\mu(x)} - \frac {\mu(x,b)} {\sd_\mu(x)^2} = \frac 1 {\sd_\mu(x)} 
\pa{\II{a=b}-\pi_\sa(b|x)}, \;\; \forall x\in\X, a,b\in\A.
\end{align*}
Inserting into the expression for the gradient of $R$ yields
\begin{align*}
\frac{\partial R(\sa)}{\partial\sa(x,a)} &= \log\frac{\pi_\sa(a|x)}{\pi_{\mu'}(a|x)} + \sum_b\frac{\pi_\sa(b|x)}{\pi_\sa(b|x)} 
\pa{\II{a=b}-\pi_\sa(b|x)}\\
 &= \log\frac{\pi_\sa(a|x)}{\pi_{\mu'}(a|x)} + 1 - \sum_b\pi_\sa(b|x) = \log\frac{\pi_\sa(a|x)}{\pi_{\mu'}(a|x)}.
\end{align*}
The optimal policy $\pi_\sa^*$ is now directly given by Equation~\eqref{eq:gradient}:
\begin{align}\label{eq:condopt}
\pi_\eta^*(a|x) = \pi_{\mu'}(a|x)\exp\pa{\eta\pa{A(x,a)-\lambda + \varphi(x,a)}}.
\end{align}
For $\sa_\eta^*$ to belong to $\Delta$, it has to satisfy Constraint~\eqref{eq:lzero}. Because of the exponent in~\eqref{eq:condopt} and 
the fact that $\sa_\eta^*(x,a)\propto\pi_\eta^*(a|x)$, $\sa_\eta^*(x,a)\geq 0$ trivially holds for any choice of $\varphi(x,a)$, implying 
that $\varphi_\eta^*(x,a)=0$ for each $(x,a)$ by complementary slackness. Since $\pi_\eta^*(a|x)=\frac {\sa_\eta^*(x,a)} {\sd_\eta^*(x)}$, 
we also obtain the following set of constraints:
\begin{align*}
\sum_a\pi_\eta^*(a|x)=\sum_a\frac {\sa_\eta^*(x,a)} {\sd_\eta^*(x)} = \frac {\sd_\eta^*(x)} {\sd_\eta^*(x)} = 1, \;\; \forall x\in\X.
\end{align*}
Inserting the expression for $\pi_\eta^*$ yields
\begin{align*}
1 &= \sum_a\pi_{\mu'}(a|x)\exp\pa{\eta A(x,a)}e^{-\eta\lambda}, \;\; \forall x\in\X.
\end{align*}
If we expand the expression for $A(x,a)$ and rearrange the terms we obtain
\begin{align}
V(x) &= \frac 1 \eta \log\sum_a\pi_{\mu'}(a|x)\exp\pa{\eta\pa{r(x,a) - \lambda + \sum_yP(y|x,a)V(y)}}, \;\; \forall 
x\in\Xw.\label{eq:condlambda}
\end{align}
The dual function is obtained by inserting the expression for $\mu^*$ into the Lagrangian:
\begin{align}
g(V,\lambda)=\mathcal{L}(\mu_\eta^*;V,\lambda)=\sum_{x,a} \mu_\eta^*(a|x) \pa{ A(x,a) - \lambda - \frac 1 \eta \log 
\frac{\pi_\eta^*(a|x)}{\pi_{\mu'}(a|x)} } + \lambda = \lambda.\label{eq:conddual}
\end{align}
Together, Equations~\eqref{eq:condlambda} and \eqref{eq:conddual} define the dual optimization problem in Proposition~\ref{prop:conddual}. 
The expression for the optimal policy in Equation~\eqref{eq:optpol} is obtained by inserting the optimal advantage function $A_\eta^*$ 
into~\eqref{eq:condopt}.

We remark that to recover the optimal stationary state-action distribution $\mu_\eta^*$, we would have to insert the expression for the 
optimal policy $\pi_\eta^*$ into Constraints~\eqref{eq:balance} and~\eqref{eq:distr}, and solve for the stationary state distribution 
$\nu_\eta^*$. However, this is not necessary since $\mu_\eta^*$ and $\nu_\eta^*$ are not required to solve the dual function or to compute 
the optimal policy.

\section{The regularized Bellman operators}\label{app:BOP}
In this section, we define the \emph{regularized Bellman operator} $T_{\pi|\pi'}$ corresponding to the policy $\pi$ and regularized with 
respect to baseline 
$\pi'$ as
\[
 T_\eta^{\pi|\pi'}[V](x) = \sum_{a} \pi(a|x) \pa{r(x,a) - \log\frac{\pi(a|x)}{\pi'(a|x)} + \sum_{x'} P(x'|x,a) V(x')} \qquad\pa{\forall x}.
\]
Similarly, we define the \emph{regularized Bellman optimality operator} $T_{*|\pi'}$ with respect to baseline $\pi'$ as
\begin{align*}
 T_\eta^{*|\pi'}[V](x) &=  \frac 1 \eta \log\sum_a\pi'(a|x)\exp\pa{\eta\pa{r(x,a) + \sum_y P(y|x,a) V(y)}} \qquad\pa{\forall x},
\end{align*}
and the \emph{regularized greedy policy} with respect to the baseline $\pi'$ as
\begin{align*}
 G_\eta^{\pi'}[V](a|x) &\propto \pi'(a|x) \exp\pa{\eta \pa{r(x,a) + \sum_y P(y|x,a) V(y) - V(x)}}.
\end{align*}
With these notations, we can define the \emph{regularized relative value iteration} algorithm with respect to $\pi'$ by the iteration
\begin{eqnarray}
 \pi_{k+1} = G_\eta^{\pi'}[V_k] \qquad V_{k+1}(x) = T_\eta^{*|\pi'}[V_k](x) -\delta_{k+1} \label{eq:VI}
\end{eqnarray}
for some $\delta_{k+1}$ lying between the minimal and maximal values of $T_\eta^{*|\pi'}[V_k]$. A common technique is to fix a reference 
state $x'$ and choose $\delta_{k+1}=T_\eta^{*|\pi'}[V_k](x')$.

Similarly, we can define the \emph{regularized policy iteration} algorithm by the iteration
\begin{eqnarray}
 \pi_{k+1} = G_\eta^{\pi'}[V_k] \qquad V_{k+1}(x) = \pa{T_\eta^{\pi_{k+1}|\pi'}}^{\infty}[V_k](x) - \delta_{k+1} \label{eq:PI},
\end{eqnarray}
with $\delta_{k+1}$ defined analogously.

For establishing the convergence of the above procedures, it is crucial to ensure that the operator $T_\eta^{*|\pi'}$ is a 
\emph{non-expansion}: 
For any value functions $V_1$ and $V_2$, we need to ensure
\[
 \norm{T_\eta^{*|\pi'}[V_1] - T_\eta^{*|\pi'}[V_2]} \le \norm{V_1 - V_2}
\]
for some norm. We state the following result claiming that the above requirement indeed holds and present the simple proof below. 
We note that analogous results have been proven several times in the literature, see, e.g., \citep{foxUAI16,AL16}.
\begin{proposition}
 $T_\eta^{*|\pi'}$ is a non-expansion for the supremum norm $\infnorm{f} = \max_x \left|f(x)\right|$.
\end{proposition}
\begin{proof}
 For simplicity, let us introduce the notation $Q_1(x,a) = r(x,a) + \sum_y P(y|x,a) V_1(y)$, with $Q_2$ defined analogously, and $\Delta 
= Q_1 - Q_2$. We have
 \begin{align*}
  T_\eta^{*|\pi'}[V_1](x) - T_\eta^{*|\pi'}[V_2](x) &=
  \frac 1 \eta \pa{\log\sum_a\pi'(a|x)\exp\pa{\eta Q_1(x,y)} - \log\sum_a\pi'(a|x)\exp\pa{\eta Q_2(x,y)}}
  \\
  &=\frac 1 \eta \pa{\log\frac{\sum_a\pi'(a|x)\exp\pa{\eta Q_1(x,y)}}{\sum_a\pi'(a|x)\exp\pa{\eta Q_2(x,y)}}}
  \\
   &=\frac 1 \eta \pa{\log\frac{\sum_a\pi'(a|x)\exp\pa{\eta \pa{Q_2(x,y) + \Delta(x,y)}}}{\sum_a\pi'(a|x)\exp\pa{\eta Q_2(x,y)}}} 
  \\
   &=\frac 1 \eta \log \sum_a p(x,a) \exp\pa{\eta \Delta(x,a)} \qquad\mbox{(with an appropriately defined $p$)}
   \\
   &\le \frac 1 \eta \log \max_a \exp\pa{\eta \Delta(x,a)}
   \\
   &= \max_a \Delta(x,a) = \max_a \sum_y P(y|x,a) \pa{V_1(y) - V_2(y)} 
   \\
   &\le \max_y \left|V_1(y) - V_2(y)\right|.
 \end{align*}
With an analogous technique, we can also show the complementary inequality
\[
 T_\eta^{*|\pi'}[V_2](x) - T_\eta^{*|\pi'}[V_1](x) \le \max_y \left|V_2(y) - V_1(y)\right|,
\]
which concludes the proof.
\end{proof}
Together with the easily-seen fact that $T_\eta^{*|\pi'}$ is continuous, this result immediately implies that $T_\eta^{*|\pi'}$ has a fixed 
point by Brouwer's fixed-point theorem. Furthermore, this insight allows us to treat the value iteration 
method~\eqref{eq:VI} as an instance of \emph{generalized value iteration}, as defined by \citet{LS96}.

We now argue that regularized value iteration converges to the fixed point of $T_\eta^{*|\pi'}$. If the initial value function $V_0$ is 
bounded, then so is $V_k$ for each $k$ since the operator $T_\eta^{*|\pi'}$ is a non-expansion. Similar to Section~\ref{sec:regmdp}, we 
assume without loss of generality that the initial reference policy $\pi_0$ has full support, i.e.~$\pi_0(a|x)>0$ for each recurrent state 
$x$ and each action $a$. Inspecting the greedy policy operator $G_\eta^{\pi'}$, it is easy to show by induction that $\pi_k$ has full 
support for each $k$. In particular, $\pi_{k+1}(a|x)$ only equals $0$ if either $\pi_k(a|x)$ equals $0$ or if the exponent $A_k(x,a)$ 
equals 
$-\infty$, which is only possible if $V_k$ is unbounded.

 Now, since $\pi_k$ has full support for each $k$, any trajectory always has a small probability of reaching a given recurrent state. We 
can 
now use a similar argument as \citet[Prop.~4.3.2]{Ber07:DPbookVol2} to show that regularized value iteration converges to the fixed point 
for $T_\eta^{*|\pi'}$.

\subsection{The proof of Lemma~\ref{lem:PD}}\label{app:PD}
 Let $\mu$ and $\mu'$ be the respective stationary distributions of $\pi$ and $\pi'$. The statement follows easily from using the 
definition of $A_{\eta}^\pi$:
 \begin{align*}
  \sum_{x,a} \mu'(x,a) A_{\eta}^\pi(x,a) &= \sum_{x,a} \mu'(x,a) \pa{r(x,a) - \frac 1 \eta \frac{\log \pi(a|x)}{\log \pi'(a|x)} - \trho(\mu) 
+ 
\sum_y P(y|x,a) V_{\eta}^\pi(y) - V_{\eta}^\pi(x)}
  \\
  &= \trho(\mu') - \trho(\mu) + \sum_{x,a} \mu'(x,a) \pa{\sum_y P(y|x,a) V_{\eta}^\pi(y) - V_{\eta}^\pi(x)}
  \\
  &= \trho(\mu') - \trho(\mu),
 \end{align*}
 where the last step follows from the stationarity of $\mu'$. \jmlrQED

\subsection{Regularized policy gradient}\label{app:relgrad}

Here we prove Lemma~\ref{lemma:polgrad} which gives the gradient of the regularized average reward $\trho_\eta(\theta)$ when the policy 
$\pi_\theta$ is parameterized on $\theta$.
Following \citet{SMSM99}, we first compute the gradient of $V_\eta^{\pi_\theta}$:
\begin{align*}
\frac {\partial V_\eta^{\pi_\theta}(x)} {\partial\theta_i} &= \sum_a \frac {\partial\pi_\theta(a|x)} {\partial\theta_i} \pa{r(x,a) - \frac 1 
\eta \log \frac{\pi_\theta(a|x)}{\pi'(a|x)} - \trho_\eta(\theta)+ \sum_y P(y|x,a) V_\eta^{\pi_\theta}(y)}\\
 &+ \sum_a\pi_\theta(a|x)\pa{- \frac 1 {\eta\pi_\theta(a|x)}\frac {\partial\pi_\theta(a|x)} {\partial\theta_i} - \frac {\partial\trho_\eta} 
{\partial\theta_i} + \sum_y P(y|x,a) \frac {\partial V_\eta^{\pi_\theta}(y)} {\partial\theta_i} }.
\end{align*}
Rearranging the terms gives us
\begin{align*}
\frac {\partial\trho_\eta} {\partial\theta_i} &= \sum_a \frac {\partial\pi_\theta(a|x)} {\partial\theta_i} \pa{r(x,a) - \frac 1 \eta \log 
\frac{\pi_\theta(a|x)}{\pi'(a|x)} - \trho_\eta(\theta)+ \sum_y P(y|x,a) V_\eta^{\pi_\theta}(y) - \frac 1 \eta}\\
 &+ \sum_a\pi_\theta(a|x) \sum_y P(y|x,a) \frac {\partial V_\eta^{\pi_\theta}(y)} {\partial\theta_i} -\frac {\partial 
V_\eta^{\pi_\theta}(x)} {\partial\theta_i}\\
 &= \sum_a \frac {\partial\pi_\theta(a|x)} {\partial\theta_i} A_\eta^{\pi_\theta}(x,a) + \sum_a\pi_\theta(a|x) \sum_y P(y|x,a) \frac 
{\partial V_\eta^{\pi_\theta}(y)} {\partial\theta_i} -\frac {\partial V_\eta^{\pi_\theta}(x)} {\partial\theta_i}.
\end{align*}
The last equality follows from the fact that since $\sum_a \partial\pi_\theta(a|x) / \partial\theta_i=0$ for each $x$, we can add any 
state-dependent constant to the multiplier of $\partial\pi_\theta(a|x)/\partial\theta_i$; adding the term $\frac 1 \eta- 
V_\eta^{\pi_\theta}(x)$ results in the given expression. Summing both sides over the stationary state distribution $\nu_{\pi_\theta}$ yields
\begin{align*}
\frac {\partial\trho_\eta} {\partial\theta_i} &= \sum_x\nu_{\pi_\theta}(x)\frac {\partial\trho_\eta} {\partial\theta_i} = \sum_{x,a} 
\nu_{\pi_\theta}(x)\frac {\partial\pi_\theta(a|x)} {\partial\theta_i} A_\eta^{\pi_\theta}(x,a)\\
 & \hspace*{2.65cm} + \sum_y\sum_{x,a}\nu_{\pi_\theta}(x)\pi_\theta(a|x)P(y|x,a) \frac {\partial V_\eta^{\pi_\theta}(y)} {\partial\theta_i} 
- \sum_x\nu_{\pi_\theta}(x)\frac {\partial V_\eta^{\pi_\theta}(x)} {\partial\theta_i}\\
 &= \sum_{x,a} \nu_{\pi_\theta}(x)\frac {\partial\pi_\theta(a|x)} {\partial\theta_i} A_\eta^{\pi_\theta}(x,a) + 
\sum_y\nu_{\pi_\theta}(y)\frac {\partial V_\eta^{\pi_\theta}(y)} {\partial\theta_i} - \sum_x\nu_{\pi_\theta}(x)\frac {\partial 
V_\eta^{\pi_\theta}(x)} {\partial\theta_i}\\
 &= \sum_{x,a} \nu_{\pi_\theta}(x)\frac {\partial\pi_\theta(a|x)} {\partial\theta_i} A_\eta^{\pi_\theta}(x,a).
\end{align*}
To conclude the proof it is sufficient to note that
\[
\mu_{\pi_\theta}(x,a)\frac {\partial \log \pi_\theta(a|x)} {\partial\theta_i} = \frac {\mu_{\pi_\theta}(x,a)} {\pi_\theta(a|x)} \frac 
{\partial \pi_\theta(a|x)} {\partial\theta_i} = \nu_{\pi_\theta}(x)\frac {\partial\pi_\theta(a|x)} {\partial\theta_i}.
\]
\jmlrQED

\subsection{The closed form of the TRPO update}\label{app:TRPO}
Here we derive the closed-form solution of the TRPO update. To do so, we first briefly summarize the mechanism of 
the algorithm. The main idea of \citet{schulman2015trust} is replacing $\rho(\mu')$ by the surrogate
\[
 L^{\pi}(\pi') = \rho(\pi) + \sum_{x} \nu_\pi(x) \sum_a \pi'(a|x) A^\pi_\infty(x,a),
\]
where $A^\pi_\infty$ is the unregularized advantage function corresponding to policy $\pi$. \footnote{This form is inspired by 
the well-known identity we state as Lemma~\ref{lem:PD}.}
Furthermore, TRPO uses the regularization term
\[
 \DD{\mu}{\mu'}{\mbox{\tiny{TRPO}}} = \sum_{x} \nu_{\mu'}(x) \sum_a \pi_{\mu}(a|x) \log \frac{\pi_{\mu}(a|x)}{\pi_{\mu'}(a|x)}.
\]
The difference between $L + D_{\mbox{\tiny{TRPO}}}$ and $\rho + D_C$ is that the approximate version ignores the impact of changing the 
policy $\pi$ on the stationary distribution. Given this surrogate objective, TRPO approximately computes the distribution
\begin{equation}\label{eq:TRPO}
 \mu_{k+1} = \argmax_{\mu\in\Delta} \ev{L_{\mu_k}(\mu) - \frac 1 \eta \DD{\mu}{\mu_k}{\mbox{\tiny{TRPO}}}}.
\end{equation}
Observing that the TRPO policy update can be expressed equivalently as
\begin{align*}
 \pi_{k+1} &= \argmax_{\pi} \ev{\sum_{x} \nu_{\mu_k}(x) \sum_a \pi(a|x) \pa{A_{\infty}^{\pi_k}(x,a) - \frac 1 \eta \log 
\frac{\pi(a|x)}{\pi_k(a|x)}}},
\end{align*}
we can see that the policy update can be expressed in closed form as
\[
 \pi_{k+1}(a|x) \propto \pi_{k}(a|x) e^{\eta A_{\infty}^{\pi_k}(x,a)}.
\]
This update then can be seen as a regularized greedy step with respect to the value function of the previous policy $\pi_k$.

\end{document}